\documentclass{article}

\PassOptionsToPackage{numbers}{natbib}
\usepackage[preprint]{neurips_2020}

\usepackage[utf8]{inputenc} % allow utf-8 input
\usepackage[T1]{fontenc}    % use 8-bit T1 fonts
\usepackage{hyperref}       % hyperlinks
\usepackage{url}            % simple URL typesetting
\usepackage{booktabs}       % professional-quality tables
\usepackage{amsfonts}       % blackboard math symbols
\usepackage{nicefrac}       % compact symbols for 1/2, etc.
\usepackage{microtype} 
\usepackage{graphicx}
\usepackage{comment}
\usepackage{amsmath,amssymb} % define this before the line numbering.
\usepackage{color}

\usepackage{amsfonts,bm,mathtools,amsthm}
\usepackage{xcolor}

\usepackage{epsfig}
\usepackage{algorithmic}
\usepackage{algorithm}
\usepackage{adjustbox}
\usepackage{subcaption}
\usepackage{wrapfig}

\title{Discriminator Contrastive Divergence: Semi-Amortized Generative Modeling by Exploring Energy of the Discriminator}

\author{%
  Yuxuan Song*$^{1}$  Qiwei Ye*$^{2}$ Minkai Xu*$^{1}$ 
  Tie-Yan Liu$^{2}$\\
  $^{1}$Shanghai Jiao Tong University $^{2}$Microsoft Research\\
  \texttt{\{songyuxuan,mkxu\}@apex.sjtu.edu.cn}, \texttt{\{qiwye,tie-yan.liu\}@microsoft.com}
}
\newcommand{\tpdv}[2]{\frac{\partial{#1}}{\partial{#2}}}

\newtheorem{theorem}{Theorem}

\newtheorem{proposition}{Proposition}

\newtheorem{lemma}{Lemma}

\def\eqref#1{equation~\ref{#1}}
\def\1{\bm{1}}

\def\gL{{\mathcal{L}}}

\newcommand{\KL}{D_{\mathrm{KL}}}

\DeclareMathOperator*{\argmin}{arg\,min}

\newcommand{\bb}[1]{{\mathbb{#1}}}

\newcommand{\defeq}{\vcentcolon=}

\begin{document}

\maketitle

\renewcommand{\thefootnote}{\fnsymbol{footnote}}
\footnotetext[1]{Equal contribution, with the order determined by flipping coins.}
\renewcommand{\thefootnote}{\arabic{footnote}}

\begin{abstract}    
    Generative Adversarial Networks (GANs) have shown great promise in modeling high dimensional data. The learning objective of GANs usually minimizes some measure discrepancy, \textit{e.g.}, $f$-divergence~($f$-GANs~\cite{nowozin2016f}) or Integral Probability Metric~(Wasserstein GANs~\cite{arjovsky2017wasserstein}). With $f$-divergence as the objective function, the discriminator essentially estimates the density ratio~\cite{uehara2016generative}, and the estimated ratio proves useful in further improving the sample quality of the generator~\cite{azadi2018discriminator,turner2018metropolis}. However, how to leverage the information contained in the discriminator of Wasserstein GANs (WGAN)~\cite{arjovsky2017wasserstein} is less explored.  In this paper, we introduce the Discriminator Contrastive Divergence, which is well motivated by the property of WGAN's discriminator and the relationship between WGAN and energy-based model. Compared to standard GANs, where the generator is directly utilized to obtain new samples, our method proposes a semi-amortized generation procedure where the samples are produced with the generator's output as an initial state. Then several steps of Langevin dynamics are conducted using the gradient of the discriminator. We demonstrate the benefits of significant improved generation on both synthetic data and several real-world image generation benchmarks.\footnote{Code is available at \url{https://github.com/MinkaiXu/Discriminator-Contrastive-Divergence}.}
\end{abstract}

\section{Introduction}

Generative Adversarial Networks (GANs)~\cite{goodfellow2014generative} proposes a widely popular way to learn likelihood-free generative models, which have shown promising results on various challenging tasks. Specifically, GANs are learned by finding the equilibrium of a min-max game between a generator and a discriminator, or a critic under the context of WGANs. Assuming the optimal discriminator can be obtained, the generator substantially minimizes some discrepancy between the generated distribution and the target distribution. 

Improving training GANs by exploring the discrepancy measure with the excellent property has stimulated fruitful lines of research works and is still an active area. Two well-known discrepancy measures for training GANs are $f$-divergence and Integral Probability Metric (IPM)~\cite{muller1997integral}. $f$-divergence is severe for directly minimization due to the intractable integral, $f$-GANs provide minimization instead of a variational approximation of $f$-divergence between the generated distribution $p_{G_\theta}$ and the target distribution $p_{\text{data}}$. The discriminator in $f$-GANs serves as a density ratio estimator~\cite{uehara2016generative}.
The other families of GANs are based on the minimization of an Integral Probability Metric (IPM). According to the definition of IPM, the critic needs to be constrained into a specific function class. When the critic is restricted to be 1-Lipschitz function, the corresponding IPM turns to the Wasserstein-1 distance, which inspires the approaches of  Wasserstein GANs~(WGANs)~\cite{miyato2018spectral,arjovsky2017wasserstein,gulrajani2017improved}.

No matter what kind of discrepancy is evaluated and minimized, the discriminator is usually discarded at the end of the training, and only the generator is kept to generate samples. A natural question to ask is whether, and how we can leverage the remaining information in the discriminator to construct a more superior distribution than simply sampling from a generator. 

Recent work ~\cite{azadi2018discriminator,turner2018metropolis} has shown that a density ratio can be obtained through the output of discriminator, and a more superior distribution can be acquired by conducting rejection sampling or Metropolis-Hastings sampling with the estimated density ratio based on the original GAN~\cite{goodfellow2014generative}.

However, the critical limitation of previous methods lies in that they can not be adapted to WGANs, which enjoy superior empirical performance over other variants. How to leverage the information of a WGAN's critic model to improve image generation remains an open problem. In this paper, we do the following to address this: 

\begin{itemize}
\vspace{-3pt}
    \item We provide a generalized view to unify different families of GANs by investigating the informativeness of the discriminators.
    \vspace{-1pt}
    \item We propose a semi-amortized generative modeling procedure so-called discriminator contrastive divergence~(DCD), which achieves an intermediate between implicit and explicit generation and hence allows a trade-off between generation quality and speed.
\vspace{-3pt}
\end{itemize}
Extensive experiments are conducted to demonstrate the efficacy of our proposed method on both synthetic setting and real-world generation scenarios, which achieves state-of-the-art performance on several standard evaluation benchmarks of image generation.

\section{Related Works}
Both empirical~\cite{arjovsky2017wasserstein} and theoretical~\cite{heusel2017gans} evidence has demonstrated that learning a discriminative model with neural networks is relatively easy, and the neural generative model(sampler) is prone to reach its bottleneck during the optimization. Hence, there is strong motivation to further improve the generated distribution by exploring the remaining information. Two recent advancements are discriminator rejection sampling(DRS)~\cite{azadi2018discriminator} and MH-GANs~\cite{turner2018metropolis}. DRS conducts rejection sampling on the output of the generator. The vital limitation that lies in the upper bound of $D_\phi$ is needed to be estimated for computing the rejection probability. MH-GAN sidesteps the above problem by introducing a Metropolis-Hastings sampling procedure with generator acting as the independent proposal; the state transition is estimated with a well-calibrated discriminator. However, the theoretical justification of both the above two methods is based on the fact that the output of discriminator needs to be viewed as an estimation of density ratio $\frac{p_{\text{data}}}{p_{G_\theta}}$. As pointed out by previous work~\cite{zhou2019lipschitz}, the output of a discriminator in WGAN~\cite{arjovsky2017wasserstein} suffers from the free offset and can not provide the density ratio, which prevents the application of the above methods in WGAN. 

Our work is inspired by recent theoretical studies on the property of discriminator in WGANs~\cite{gulrajani2017improved,zhou2019lipschitz}. \cite{tanaka2019discriminator} proposes discriminator optimal transport~(DOT) to leverage the optimal transport plan implied by WGANs' discriminator, which is orthogonal to our  method. 
Moreover, turning the discriminator of WGAN into an energy function is closely related to the amortized generation methods in the context of the energy-based model (EBM)~\cite{kim2016deep,zhao2016energy,kumar2019maximum} where a separate network is proposed to learn to sample from the partition function in~\cite{finn2016connection}. Recent progress~\cite{song2019generative,du2019implicit} in the area of EBM has shown the feasibility of generating high dimensional data with Langevin dynamics. From the perspective of EBM, our proposed method can be seen as an intermediary between an amortized generation model and an implicit generation model, \emph{i.e.}, a semi-amortized generation method, which allows a trade-off between speed and flexibility of generation.  With a similar spirit, \cite{grathwohl2019your}  also illustrates the potential connection between neural classifier and energy-based model in supervised and semi-supervised scenarios.  
\section{Preliminaries}
\subsection{Generative Adversarial Networks}
\label{GANs}
Generative Adversarial Networks (GANs)~\cite{goodfellow2014generative} is an implicit generative model that aims to fit an empirical data distribution $p_{\text{data}}$ over sample space $\mathcal{X}$. The generative distribution $p_{G_\theta}$ is implied by a generated function $G_\theta$, which maps latent variable $Z$ to sample $X$, \emph{i.e.}, $G_\theta : \mathcal{Z} \xrightarrow{} \mathcal{X}$. Typically, the latent variable $Z$ is distributed on a fixed  prior distribution $p(z)$. With i.i.d samples available from $p_{G_\theta}$ and $p_{\text{data}}$, the GAN typically learns the generative model through a min-max game between a discriminator $D_\phi$ and a generator $G_\theta$:
\begin{equation}
    \label{gan_obj}
     \min _{\theta} \max _{\phi} \mathbb{E}_{\boldsymbol{x} \sim P_{\text {data }}}\left[r(D_{\phi}(\boldsymbol{x}))\right]-\mathbb{E}_{\boldsymbol{x} \sim p_{G_{\theta}}}\left[m(D_{\phi}(\boldsymbol{x}))\right].
\end{equation}
With $r$ and $m$ as the function $r(x) = m(x) = x$ and the $D_\phi(x)$ is constrained as 1-Lipschitz function, the Eq.~\ref{gan_obj} yields the WGANs objective which essentially minimizes the Wasserstein distance between $p_{\text{data}}$ and $p_{G_\theta}$. 
With $r(x)=x$ and $m(x)$ as the Fenchel conjugate\cite{hiriart2012fundamentals} of a convex and lower-semicontinuous function, the objective in Eq.~\ref{gan_obj} approximately minimize a variational estimation of $f$-divergence\cite{nowozin2016f} between $p_{\text{data}}$ and $p_{G_\theta}$.

\subsection{Energy Based Model and MCMC basics}
\label{A}
The energy-based model tends to learn an unnormalized probability model implied by an energy function $E_\theta(x)$ to prescribe the ground truth data distribution $p_{\text{data}}$. The corresponding normalized density function is:
% We consider data samples denoted as $x$ from the set $\gX$. Given a dataset with empirical distribution $p(x)$, we are interested in learning an (unnormalized) energy function $E_\theta(x)$ that prescribes the following normalized distribution:
\begin{align}
    \label{partition_function}
    q_\theta(x) = \frac{e^{-E_\theta(x)}}{Z_{\theta}},~~~~~ Z_\theta = \int e^{-E_\theta(x)} \mathrm{d} x,
\end{align}
where $Z_\theta$ is so-called normalization constant.  
The objective of training an energy-based model with maximum likelihood estimation is as:
\begin{align}
    \label{mle}
    \gL_{\mathrm{MLE}}(\theta; p) \defeq -\bb{E}_{x \sim p_{\text{data}}(x)}\left[\log q_\theta(x)\right].
\end{align}
The estimated gradient with respect to the MLE objective is as follows:
\begin{align}
\label{cd}
    & \nabla_\theta \gL_{\mathrm{MLE}}(\theta; p) \\ \nonumber
    % =\ & \nabla_\theta \bb{E}_{x \sim p_{\text{data}}(x)}\left[-\log q_\theta(x)\right] \\ \nonumber
    % =\ & \nabla_\theta \bb{E}_{x \sim p_{\text{data}}(x)}[E_\theta(x)] + \nabla_\theta \log Z_\theta \\ \nonumber
    =\ & \nabla_\theta \bb{E}_{x \sim p_{\text{data}}(x)}[E_\theta(x)] - \frac{\int e^{-E_\theta(x)} \nabla_\theta E_\theta(x) \mathrm{d} x}{Z_\theta} \\ \nonumber
    =\ & \bb{E}_{x \sim p_{\text{data}}(x)}[\nabla_\theta E_\theta(x)] - \bb{E}_{x \sim q_\theta(x)}[\nabla_\theta E_\theta(x)]\nonumber\label{eq:contrastive-divergence}.
\end{align}
The above method for gradient estimation in Equation~\ref{cd} is called contrastive divergence~(CD). 
Furthermore, we define the \emph{score} of distribution with density function $p(x)$ as $\nabla_x \log p(x)$. We can immediately conclude that $\nabla_x \log q_\theta(x) = \nabla E_\theta(x)$, which does not depend on the intractable $Z_\theta$.
% \yuxuan{need to find a position for this.}

% \subsection{MCMC basics}\qiwye{i think is section are unnecessary}\yuxuan{will fix}
Markov chain Monte Carlo is a powerful framework for drawing samples from a given distribution. An MCMC is specified by a transition kernel $\mathcal{K}(x^{\prime}|x)$ which corresponds to a unique stationary distribution $p$, \emph{i.e.},
\begin{align*}
    q=p \quad \Leftrightarrow \quad q(x)=\int q\left(x^{\prime}\right) \mathcal{K}\left(x | x^{\prime}\right) d x^{\prime}, \quad \forall x.
\end{align*}
More specifically, MCMC can be viewed as drawing $x_0$ from the initial distribution $x_0$ and iteratively get sample $x_t$ at the $t$-th iteration by applied the transition kernel on the previous step, \emph{i.e.}, $x_t|x_{t-1} \sim \mathcal{K}(x_t|x_{t-1})$. Following \cite{li2017approximate}, we formalized the distribution $q_t$ of $z_t$ as obtained by a fixed point update of form $q_{t}(x) \leftarrow \mathcal{K} q_{t-1}(x)$, and $\mathcal{K} q_{t-1}(x)$:
\begin{align*}
  \mathcal{K} q_{t-1}(x):=\int q_{t-1}\left(x^{\prime}\right) \mathcal{K}\left(x | x^{\prime}\right) d x^{\prime}.
\end{align*}
As indicated by the standard theory of MCMC, the following monotonic property is satisfied:
\begin{equation}
\label{MCMC_monotic}
        \KL(q_t||p) \leq \KL(q_{t-1}||p).
\end{equation}
And $q_t$ converges to the stationary distribution $p$ as $t \rightarrow \infty$.

\section{Methodology}
\label{sec:methodology}
\subsection{Informativeness of Discriminator}\label{sec::iod}
In this section, we seek to investigate the following questions:
\begin{itemize}
\vspace{-3pt}
    \item What kind of information is contained in the discriminator of different kinds of GANs?
    \vspace{-1pt}
    \item Why and how can the information be utilized to further improved the quality of generated distribution?
\vspace{-3pt}
\end{itemize}
We discuss the discriminator of $f$-GANs, and WGANs, respectively, in the following.

\subsubsection{$f$-GAN Discriminator}
\label{f-GAN-d}
$f$-GAN~\cite{nguyen2010estimating} is based on the variational estimation of $f$-divergence~\cite{ali1966general} with only samples from two distributions available:
%  proposes a variational estimation of $f$-divergence between two probability measure 
\begin{theorem}\cite{nguyen2010estimating}
\label{fgan}
With Fenchel Duality, the variational estimation of $f$-divergence can be illustrated as follows:
\begin{align}
&D_{f}(P \| Q) \\
=\ &\int_{\mathcal{X}} q(x) \sup _{t \in \text { dom }_{f^{*}}}\left\{\left(t\frac{p(x)}{q(x)}-f^{*}(t)\right) \mathrm{d} x\right\} \nonumber\\  
\geq\ & \sup _{T \in \mathcal{T}}\left(\int_{\mathcal{X}} p(x) T(x) \mathrm{d} x-\int_{\mathcal{X}} q(x) f^{*}(T(x)) \mathrm{d} x\right) \nonumber\\ =\ &\sup _{T \in \mathcal{T}}\left(\mathbb{E}_{x \sim P}[T(x)]-\mathbb{E}_{x \sim Q}\left[f^{*}(T(x))\right]\right), \nonumber
\end{align}
where the $\mathcal{T}$ is the arbitrary class of function and $f^{*}$ denotes the Fenchel conjugate of $f$. And the supremum is achieved only when $T^{*}(x)=f^{\prime}\left(\frac{p(x)}{q(x)}\right)$, \emph{i.e.} $\frac{p(x)}{q(x)} = \frac{\partial f^{*}}{\partial T}(T^{*}(x))$.
\end{theorem}
In $f$-GAN~\cite{nowozin2016f}, the discriminator $D_\phi$ is actually the function $T$ parameterized with neural networks. Theorem.~\ref{fgan} indicates the density ratio estimation view of $f$-GAN's discriminator, as illustrated in \cite{uehara2016generative}. More specifically, the discriminator$D_\phi$ in $f$-GAN is  optimized to estimate a statistic related to the density ratio between $p_{\text{data}}$ and $p_{G_\theta}$, \emph{i.e.} $\frac{p_{\text{data}}}{p_{G_\theta}}$, and the  $\frac{p_{\text{data}}}{p_{G_\theta}}$ can be acquired easily with $D_\phi$. For example, in the original GANs~\cite{goodfellow2014generative}, the corresponding $f$ in $f$-GAN literature is $f(x) = x\log x - (x+1)\log(x+1) + 2\log2 $. Assuming the discriminator is trained to be optimal, the output is $D_\phi(x) = \frac{p_{\text{data}}}{p_{\text{data}}+p_{G_\theta}}$, and we can get the density ratio $\frac{p_{\text{data}}}{p_{G_\theta}} = \frac{D_\phi(x)}{1-D_\phi(x)}$. However, it should be noticed that the discriminator is hard to reach the optimality. In practice, without loss of generality, the density ratio implied by a sub-optimal discriminator can be seen as the density ratio between an  implicitly defined distribution $p^{*}$ and the generated distribution $p_{G_\theta}$. It has been studied both theoretically and empirically in the context of GANs~\cite{arjovsky2017wasserstein,heusel2017gans,hjelm2017boundary}, with the same inductive bias, that learning a discriminative model is more accessible than a generative model. Based on the above fact, the rejection-sampling based methods are proposed to use the estimated density ratio, \emph{e.g.}, $\frac{D_\phi(x)}{1-D_\phi(x)}$ in original GANs, to conduct rejection sampling\cite{azadi2018discriminator} or Metropolis-Hastings sampling\cite{turner2018metropolis} based on generated distribution $p_{G_\theta}$. These methods radically modify the generated distribution $p_{G_\theta}$ to $p^{*}$, the improvement in empirical performance as shown in \cite{azadi2018discriminator,turner2018metropolis} demonstrates that we can construct a superior distribution  $p^{*}$  to prescribe the empirical distribution $p_{\text{data}}$ by involving the remaining information in discriminator. 

\begin{figure*}[!t]\label{fig::overview}
	\centering
    \includegraphics[width=1.0\linewidth]{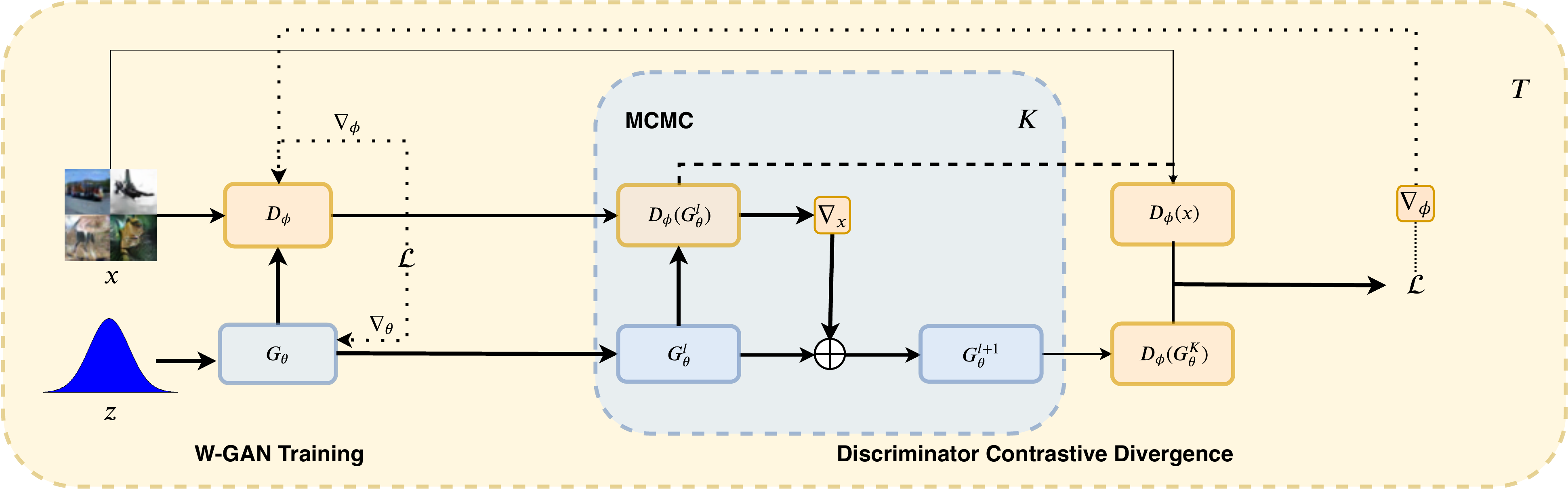} 
    \caption{\label{fig:dcd} Discriminator Contrastive Divergence: 
    After WGAN training, a fine-tuning for critics can be conducted with several MCMC steps, which leverages the gradient of discriminator by Langevin dynamics;
    after the fine-tuning, the discriminator could be viewed as a superior distribution of $p_{G_\theta}$, hence sampling from $p_{G_\theta}$ can be implemented using the same Langevin dynamics as described in \ref{DCD}.
    }
    \vspace{-5pt}
\end{figure*}

\subsubsection{WGAN Discriminator}\label{sec::wgan-critic}
Different from $f$-GANs, the objective of WGANs is derived from the Integral Probability Metric, and the discriminator can not naturally be derived as an estimated density ratio. Before leveraging the remaining information in the discriminator, the property of the discriminator in WGANs needs to be investigated first. We introduce the primal problem implied by WGANs objective as follows:

Let $\pi$ denote the joint probability for transportation between $P$ and $Q$, which satisfies the marginality conditions,
\begin{equation}\label{wgan_p}
       \int d \boldsymbol{y} \pi(\boldsymbol{x}, \boldsymbol{y})=p(\boldsymbol{x}), \quad  \int d \boldsymbol{x} \pi(\boldsymbol{x}, \boldsymbol{y})=q(\boldsymbol{y})
\end{equation}
The primal form first-order Wasserstein distance $W_1$ is defined as:
\begin{align*}
W_{1}\left(\mathcal{P}, \mathcal{Q}\right)=\inf _{\pi \in \Pi\left(\mathcal{P}, \mathcal{Q}\right)} \mathbb{E}_{(x, y) \sim \pi}[\|x- y\|_{2}]
\end{align*}
the objective function of the discriminator in Wasserstein GANs is the Kantorovich-Rubinstein duality of  Eq.~\ref{wgan_p}, and the optimal discriminator has the following property\cite{gulrajani2017improved}:
\begin{theorem}
\label{gradient_direction}
Let $\pi^{*}$ as the optimal transport plan in Eq.~\ref{wgan_p} and $x_{t}=t x+(1-t) y$ with $0 \leq t \leq 1$. With the optimal discriminator $D_\phi$ as a differentiable function and $\pi^{*}(x,x) = 0$ for all $x$, then it holds that:
\begin{align*}
    \mathrm{P}_{(x, y) \sim \pi^{*}}\left[\nabla_{x_{i}} D_\phi^{*}\left(x_{t}\right)=\frac{y-x}{\|y-x\|}\right]=1
\end{align*}
\end{theorem}

Theorem.~\ref{gradient_direction} states that for each sample $x$ in the generated distribution $p_{G_\theta}$, the gradient on the $x$ directly points to a sample $y$ in the $p_{\text{data}}$, where the $(x,y)$ pairs are consistent with the optimal transport plan $\pi ^{*}$. All the linear interpolations $x_t$ between $x$ and $y$ satisfy that $\nabla_{x_{k}} D_\phi^{*}\left(x_{t}\right)=\frac{y-x}{\|y-x\|}$. 
It should also be noted that similar results can also be drawn in some variants of WGANs, whose loss functions may have a slight difference with standard WGAN~\cite{zhou2019lipschitz}. 
For example, the SNGAN uses the hinge loss during the optimization of the discriminator, \textit{i.e.}, $r(\cdot)$ and $g(\cdot)$ in Eq.~\ref{gan_obj} is selected as $\max (0,-1-u)$ for stabilizing the  training procedure. 
% cite{zhou2019lipschitz} proposed a generalized version of Theorem.~\ref{gradient_direction}, which shows that using hinge loss during optimization do not break the gradient property of discriminator. 
We provide a detailed discussion on several surrogate objectives in Appendix.~\ref{app:obj}.

The above  property of discriminator in WGANs can be interpreted as that given a sample $x$ from generated distribution $p_{G_\theta}$ we can obtain a corresponding $y$ in data distribution $p_{\text{data}}$ by directly conducting gradient decent with the optimal discriminator $D_\phi^{*}$:
\begin{equation}
\label{add grad}
        y = x + w_x*\nabla_{x} D_\phi^{*},\quad w_x \geq 0
\end{equation}
It seems to be a simple and appealing solution to improve $p_{G_\theta}$ with the guidance of discriminator $D_\phi$. However, the following issues exist: 

1) there is no theoretical indication on how to set $w_x$ for each sample $x$ in generated distribution. We noticed that a concurrent work~\cite{tanaka2019discriminator} introduce a search process called Discriminator Optimal Transport(DOT) by finding the corresponding $y^{*}$ through the following:
\begin{equation}
\label{dot}
    y_x =\underset{\boldsymbol{y}}{\arg \min }\left\{\|\boldsymbol{y}-\boldsymbol{x}\|_{2}-D_\phi^{*}(\boldsymbol{y})\right\}
\end{equation}
However, it should be noticed that Eq.~\ref{dot} has a non-unique solution. As indicated by Theorem~\ref{gradient_direction}, all points on the connection between $x$ and $y$ are valid solutions. We further extend the fact into the following theorem:
\begin{theorem}
\label{opt_fail}
    With the $\pi ^{*}$ and $D_\phi^{*}$ as the optimal solutions of the primal problem in Eq.~\ref{wgan_p} and  Kantorovich-Rubinstein duality of Eq.~\ref{wgan_p}, the distribution $p_{ot}$ implied by the generated distribution $p_{G_\theta}$and the discriminator $D_\phi^{*}$ is defined as($y_x$ is defined in Eq.~\ref{dot}):
    \begin{align*}
       p_{ot}(\boldsymbol{y})=\int d \boldsymbol{x} \delta(y-y_x) p_{G_\theta}(\boldsymbol{x})
    \end{align*}
    when $p_{\text{data}} \neq p_{G_\theta}$, there exists infinite numbers of $p_{ot}$ with $p_{\text{data}}$ as a special case. 
\end{theorem}
Theorem~\ref{opt_fail} provides a theoretical justification for the poor empirical performance of conducting DOT in the sample space, as shown in their paper.

2) Another problem lies in that samples distributed outside the generated distribution ($p_{G_\theta}$) are never explored during training, which results in much adversarial noise during the gradient-based search process, especially when the sample space is high dimensional such as real-world images. 

To fix the issues mentioned above in leveraging the information of discriminator in Wasserstein GANs, we propose viewing the discriminator as an energy function. With the discriminator as an energy function, the stationary distribution is unique, and Langevin dynamics can approximately conduct sampling from the stationary distribution. Due to the monotonic property of MCMC, there will not be issues like setting $w_x$ in Eq.~\ref{add grad}. Besides, the second issue can also be easily solved by fine-tuning the energy spaces with contrastive divergence.  In addition to the benefits illustrated above, if the discriminator is an energy function, the samples from the corresponding energy-based model can be obtained through Langevin dynamics by using the gradients of the discriminator which takes advantage of the  property of discriminator as shown in Theorem~\ref{gradient_direction}. With all the facts as mentioned above, there is strong motivation to explore further and bridge the gap between discriminator in WGAN and the energy-based model. 

\subsection{Semi-Amortized Generation with Langevin Dynamics}
\label{sec::energyofcritic}
We first introduce the Fenchel dual of the intractable partition function $Z_\theta$ in Eq.~\ref{partition_function}:
\begin{theorem}\cite{wainwright2008graphical}
    With $H(q) = - \int q(x)\log q(x) dx$, the Fenchel dual of log-partition $Z_\theta$ is as follows:
    \begin{align}
        A(E_\theta)=\max _{q \in \mathcal{P}}\langle q(x), E_\theta(x)\rangle+ H(q),
    \end{align}
    where $\mathcal{P}$ denotes the space of distributions, and $\langle q(x), E_\theta(x)\rangle =\int  E_\theta(x)q(x) dx $.
\end{theorem}
We put the Fenchel dual of $A(E_\theta)$ back into the MLE objective in Eq.~\ref{mle}, we achieve the following min-max game formalization for training energy-based model based on MLE:
\begin{equation}
    \label{fenchel_mle}
     \min _{q \in \mathcal{P}} \max _{E_\theta \in \mathcal{E}} \underbrace{\mathbb{E}_{\boldsymbol{x} \sim P_{\text {data }}}\left[E_\theta(\boldsymbol{x})\right]-\mathbb{E}_{\boldsymbol{x} \sim q}\left[E_\theta(\boldsymbol{x})\right]}_\text{WGAN's objective for critic} -\underbrace{H(q)}_{\text{entropy regularization}}.
\end{equation}

\begin{algorithm}[t]
  \caption{Discriminator Contrastive Divergence}
  \label{DCD}
  \begin{algorithmic}[1]
  \STATE {\bfseries Input:} Pretrained generator $G_\theta$, discriminator $D_\phi$.
  \STATE Set the step size $\epsilon$, the length of MCMC steps $K$ and the total iterations $T$.
  \FOR{iteration $i = 1, \cdots, T$}
  \STATE Sample a batch of data samples $\{x_t\}_{t=1}^{m}$ for empirical data distribution $p_{\text{data}}$ and $\{z_t\}_{t=1}^{m}$ for the prior distribution $p(z)$.\\
   \FOR{iteration $l = 1, \cdots, K$}
     \STATE \textbf{Pixel Space:} $G_\theta(z_t)^l = G_\theta(z_t)^{l-1} -\frac{\epsilon}{2} \nabla_{x} D_{\phi}\left(G_\theta(z_t)^{l-1}\right)+\sqrt{\epsilon}\omega, \omega \sim \mathcal{N}(0, \mathcal{I})$ \OR
     \STATE \textbf{Latent Space:} $z_t^l = z_t^{l-1} -\frac{\epsilon}{2} \nabla_{z} D_{\phi}\left(G_\theta(z_t)^{l-1}\right)+\sqrt{\epsilon}\omega, \omega \sim \mathcal{N}(0, \mathcal{I})$
   \ENDFOR
   \STATE Optimized the following objective w.r.t. $\phi$:
   \STATE \textbf{Pixel Space:} $L = \frac{1}{m}\sum_{t}(D_\phi(x_t)-D_\phi(G_\theta(z_t)^K))$ \OR
   \STATE \textbf{Latent Space}: $L = \frac{1}{m}\sum_{t}(D_\phi(x_t)-D_\phi(G_\theta(z_t^K)))$
  \ENDFOR
\end{algorithmic}
\end{algorithm}

The Fenchel dual view of MLE training in the energy-based model explicitly illustrates the gap and connection between the WGAN and Energy based model. If we consider the dual distribution $q$ as the generated distribution $p_{G_\theta}$, and the $D_\phi$ as the energy function $E_\theta$. The duality form for training energy-based models is essentially the WGAN's objective with the entropy of the generator is regularized.

Hence to turn the discriminator in WGAN into an energy function, we may conduct several fine-tuning steps, as illustrated in Eq.~\ref{fenchel_mle}. Note that maximizing the entropy of the $p_{G_\theta}$ is indeed a challenging task, which needs to either use a tractable density generator, \emph{e.g.}, normalizing Flows~\cite{realnvp}, or maximize the mutual information between the latent variable $\mathcal{Z}$ and the corresponding $G_\theta (Z)$ when the $G_\theta$ is a deterministic mapping. However, instead of maximizing the entropy of the generated distribution $p_{G_\theta}$ directly, we derive our method based on the following fact: 
\begin{proposition}\cite{kim2016deep}\label{prop1}
    Update the generated distribution $p_{G_\theta}$ according to the gradient estimated through Equation.~\ref{fenchel_mle}, essentially minimized the Kullback–Leibler (KL) divergence between $p_{G_\theta}$ and the distribution $p_{D_\phi}$, which refers to the distribution implied by using $D_\phi$ as the energy function, as illustrated in Eq.~\ref{partition_function}, \emph{i.e.} $\KL(p_{G_\theta}||p_{D_\phi})$.
\end{proposition}

To avoid the computation of $H(p_{G_\theta})$, motivated by the monotonic property of MCMC, as illustrated in Eq.~\ref{MCMC_monotic}, we propose Discriminator Contrastive Divergence (DCD), which replaces the gradient-based optimization on $q$($p_{G_\theta}$) in Eq.~\ref{fenchel_mle} with several steps of MCMC for finetuning the critic in WGAN into an energy function. To be more specific, we use Langevin dynamics\cite{teh2003energy} which leverages the gradient of the discriminator to conduct sampling:
\begin{equation}
\label{langevin}
x_{k}=x_{k-1}-\frac{\epsilon}{2} \nabla_{x} D_{\phi}\left(x_{k-1}\right)+\sqrt{\epsilon}\omega, \omega \sim \mathcal{N}(0, \mathcal{I}),
\end{equation}
Where $\epsilon$ refers to the step size. The whole finetuning procedure is illustrated in Algorithm~\ref{DCD}. 
The GAN-based approaches are implicitly constrained by the  dimension of the latent noise, which is based on a widely applied assumption that the high dimensional data, \emph{e.g.}, images, actually distribute on a relatively low-dimensional manifold. Apart from searching the reasonable point in the data space, we could also find the lower energy part of the latent manifold by conducting Langevin dynamics in the latent space which are more stable in practice, \emph{i.e.}:
\begin{equation}
\label{langevin_Z}
z_t^l = z_t^{l-1} -\frac{\epsilon}{2} \nabla_{z} D_{\phi}\left(G_\theta(z_t)^{l-1}\right)+\sqrt{\epsilon}\omega, \omega \sim \mathcal{N}(0, \mathcal{I}).
\end{equation}

Ideally, the proposal should be accepted or rejected according to the Metropolis–Hastings algorithm:
\begin{equation}
\label{mhrej}
\alpha:=\min \left\{1, \frac{D_\phi \left(x_{k}\right) q\left(x_{k-1} | x_{k}\right)}{D_\phi \left(x_{k-1}\right) q\left(x_{k} | x_{k-1}\right)}\right\},
\end{equation}
where $q$ refers to the proposal which is defined as:
\begin{equation}
q\left(x^{\prime} | x\right) \propto \exp \left(-\frac{1}{4 \tau}\left\|x^{\prime}-x-\tau \nabla \log \pi(x)\right\|_{2}^{2}\right).
\end{equation}
In practice, we find the rejection steps described in Eq.~\ref{mhrej} do not boost performance. For simplicity, following \cite{song2019generative,du2019implicit}, we apply Eq.~\ref{langevin} in experiments as an approximate version. 

After fine-tuning, the discriminator function can be approximated seen as an unnormalized probability function, which implies a unique distribution $p_{D_\phi}$. And similar to the $p_{*}$ implied in the rejection sampling-based method, it is reasonable to assume that $p_{D_\phi}$ is a superior distribution of $p_{G_\theta}$. Sampling from $p_{D_\phi}$ can be implemented through the Langevin dynamics, as illustrated in Eq.~\ref{langevin} with $p_{G_\theta}$ serves as the initial distribution.

\section{Experiments}

\begin{figure*}[!t]
	\centering
	\begin{subfigure}{0.32\linewidth}
    	\includegraphics[width=1.0\columnwidth]{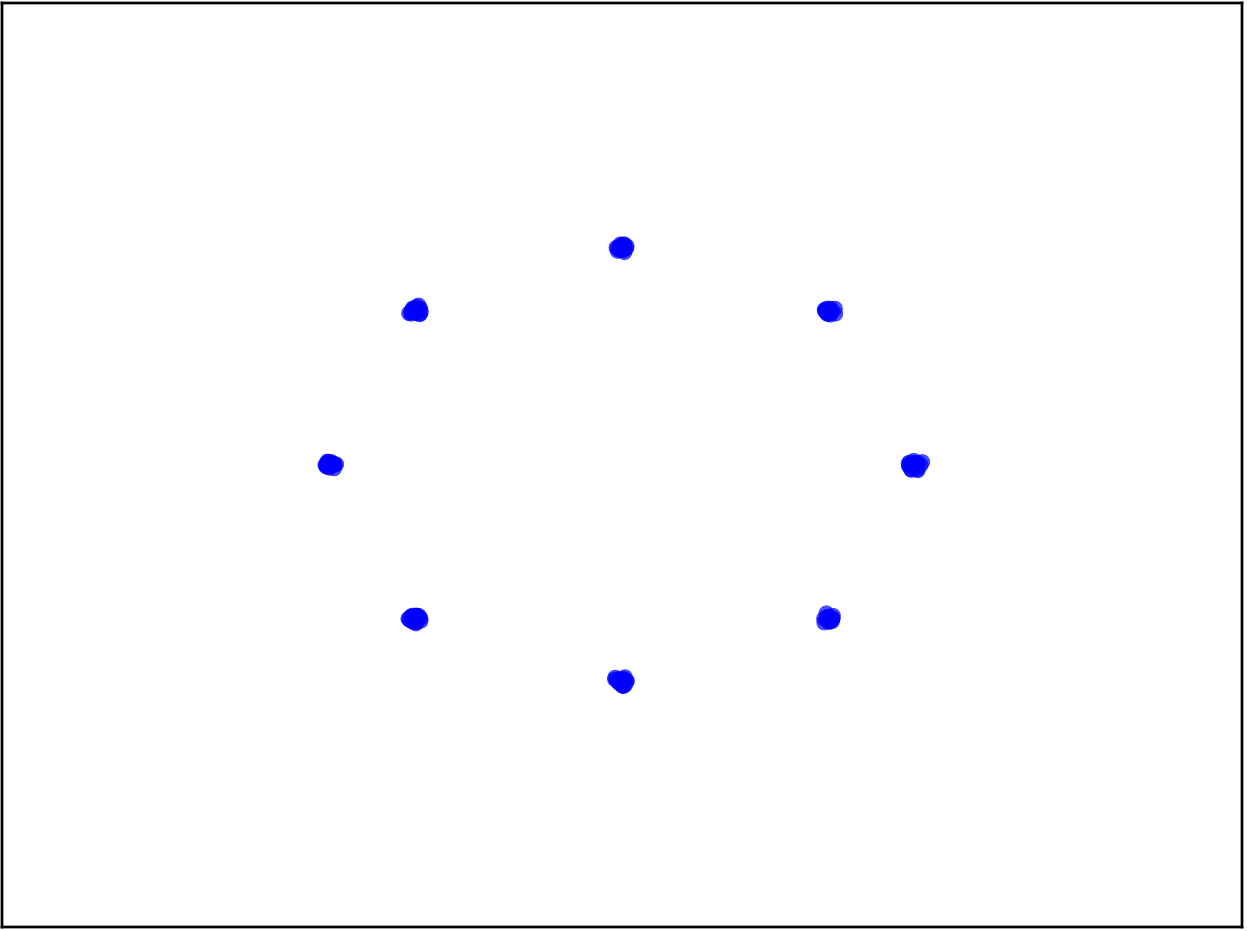} 
    \end{subfigure}
    \begin{subfigure}{0.32\linewidth}
		\includegraphics[width=1.0\columnwidth]{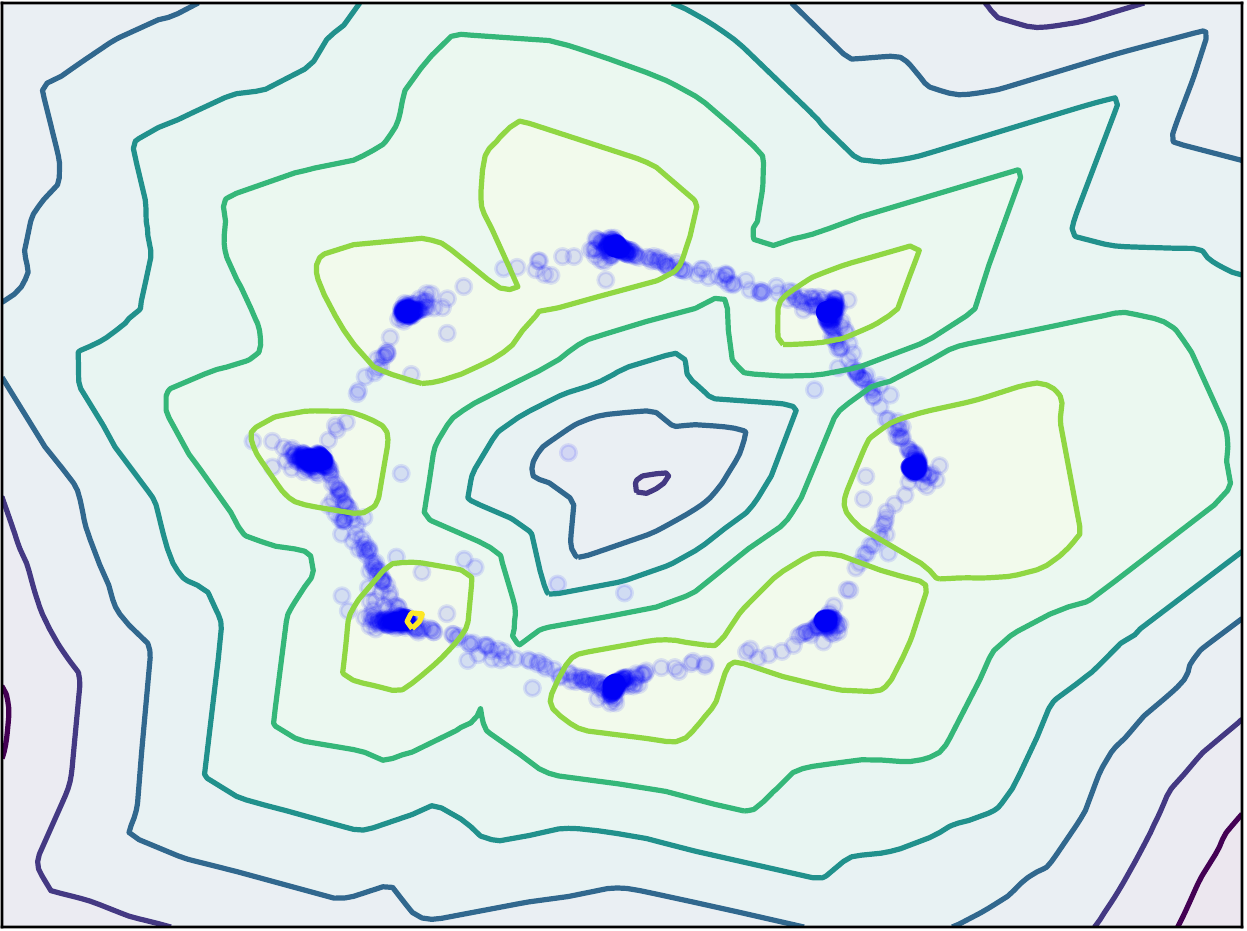}
    \end{subfigure}
    \begin{subfigure}{0.32\linewidth}
		\includegraphics[width=1.0\columnwidth]{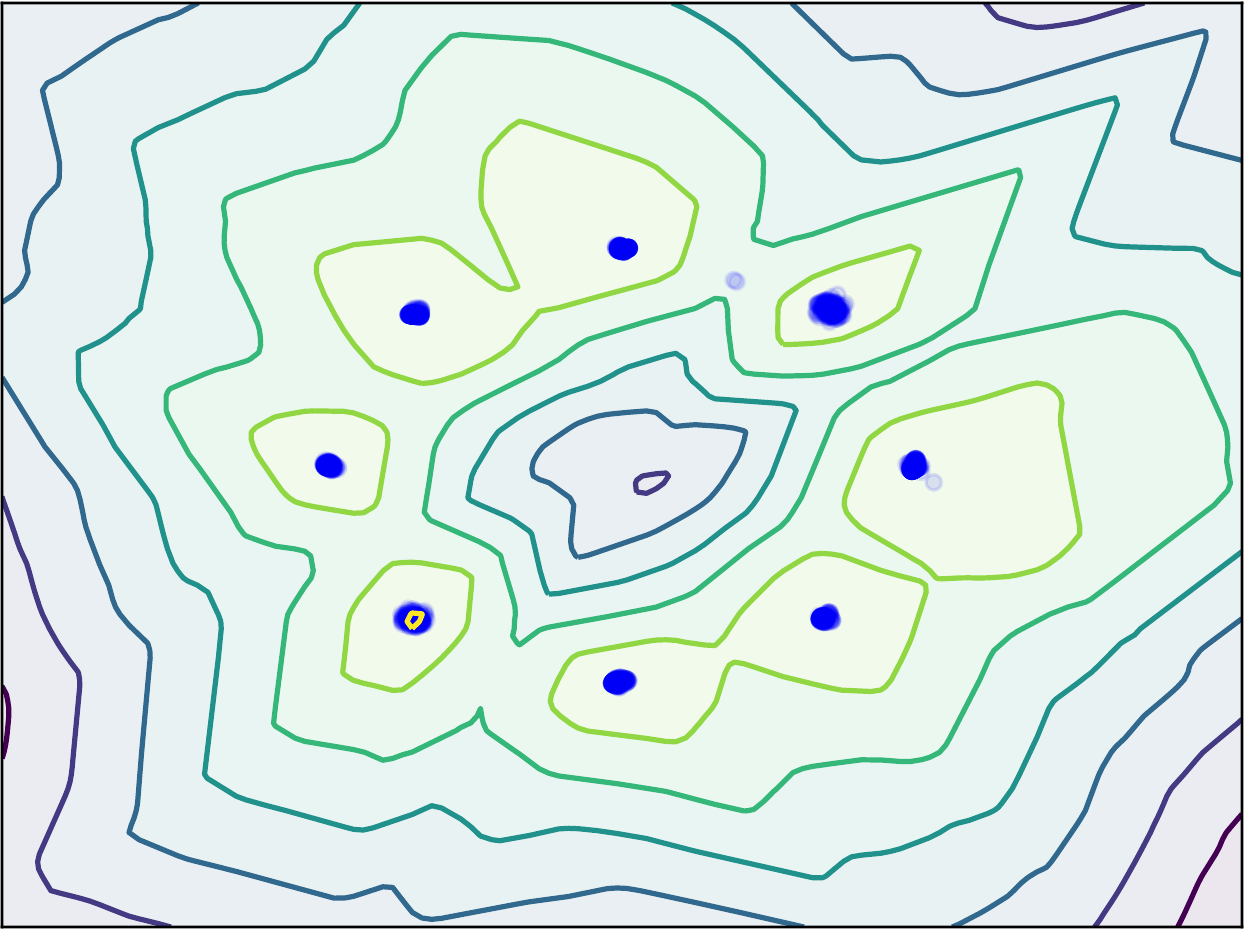}
    \end{subfigure}
    \begin{subfigure}{0.32\linewidth}
		\includegraphics[width=1.0\columnwidth]{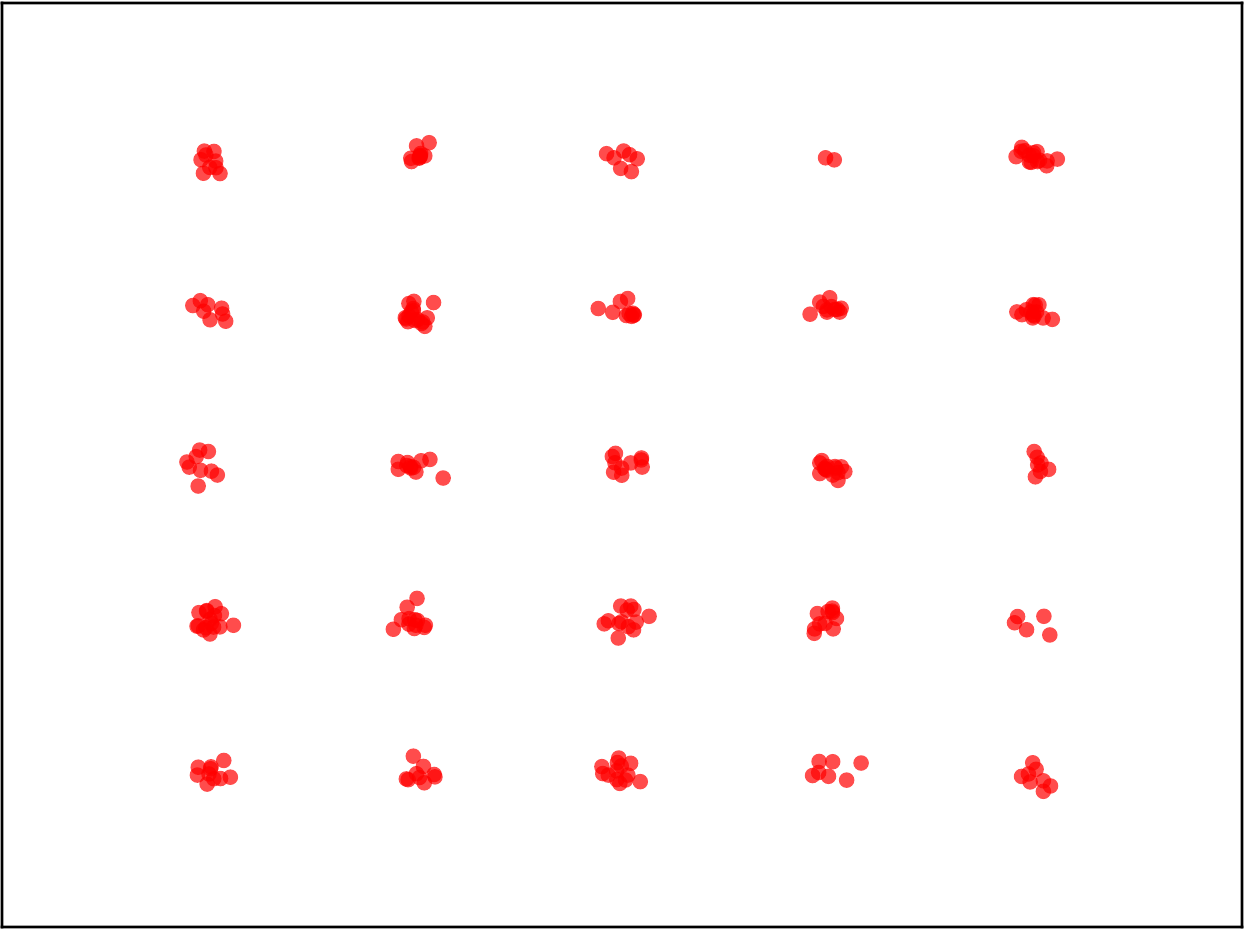} 
        \caption{Target distribution}
        \label{subfig:toy:true}
    \end{subfigure}
    \begin{subfigure}{0.32\linewidth}
		\includegraphics[width=1.0\columnwidth]{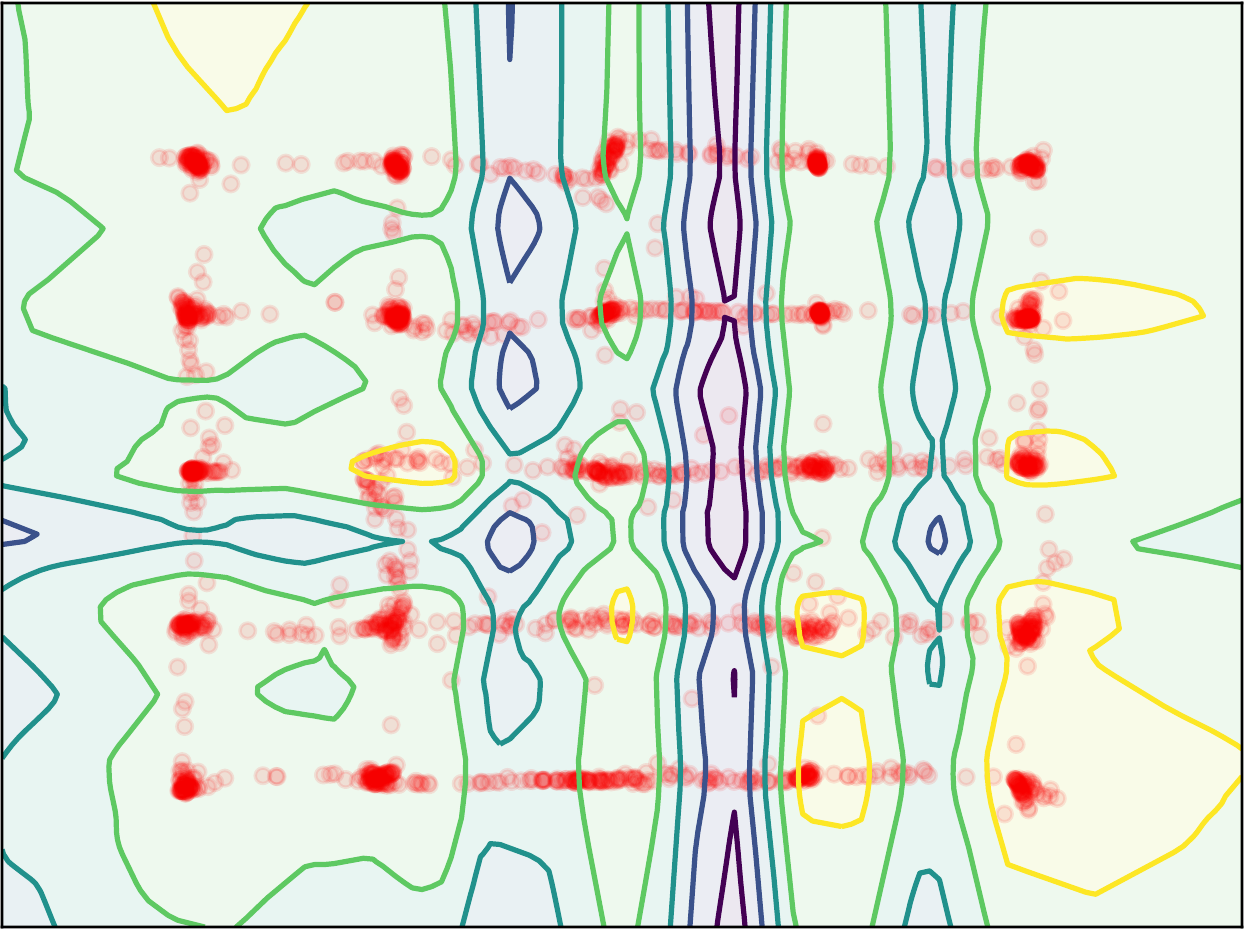}
        \caption{SNGAN}
        \label{subfig:toy:sngan}
    \end{subfigure}
    \begin{subfigure}{0.32\linewidth}
		\includegraphics[width=1.0\columnwidth]{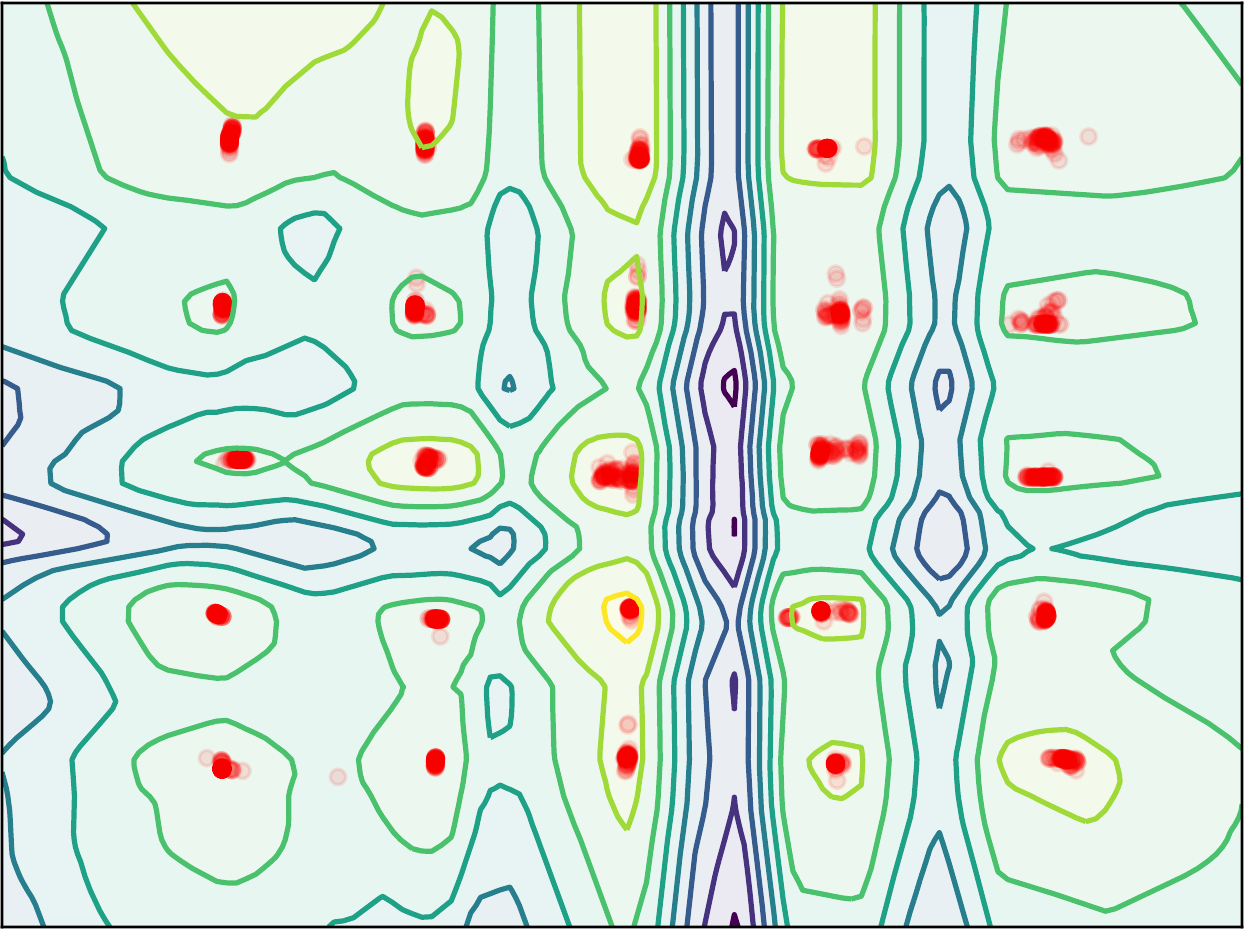}
        \caption{SNGAN-DCD}
        \label{subfig:toy:sngan-dcd}
    \end{subfigure}
    \caption{Density modeling on synthetic distributions. \textbf{Top}: 8 Gaussian distribution. \textbf{Bottom}: 25 Gaussian distribution. \textbf{Left}: Distribution of real data. \textbf{Middle}: Distribution defined by the generator of SNGAN. The surface is the level set of the critic. Yellow corresponds to higher value while purple corresponds to lower. \textbf{Right:} Distribution defined by the SNGAN-DCD. The surface is the level set of the proposed energy function.}
    \vspace{-5pt}
    \label{fig:toy}
\end{figure*}

In this section, we conduct extensive experiments on both synthetic data and real-world images to demonstrate the effectiveness of our proposed method. The results show that taking the optionally fine-tuned Discriminator as the energy function and sampling from the corresponding $p_{D_\phi}$ yield stable improvement over the WGAN implementations.

\subsection{Synthetic Density Modeling}
Displaying the level sets is a meaningful way to study learned critic. Following the \cite{azadi2018discriminator,gulrajani2017improved}, we investigate the impacts of our method on two challenging low-dimensional synthetic settings: twenty-five isotropic Gaussian distributions arranged in a grid and eight Gaussian distributions arranged in a ring (Fig.~\ref{subfig:toy:true}). For all different settings, both the generator and the discriminator of the WGAN model are implemented as neural networks with four fully connected layers and Relu activations. The Lipschitz constraint is restricted through spectral normalization~\cite{miyato2018spectral}, while the prior is a two-dimensional multivariate Gaussian with a mean of $0$ and a standard deviation of $1$.
% And the model is trained with standard WGAN loss, \emph{i.e} Eq.~\ref{GANs} with $r(x) = m(x) = x$.

To investigate whether the proposed Discriminator Contrastive Divergence is capable of tuning the distribution induced by the discriminator as desired energy function, \emph{i.e.} $p_{D_\phi}$, we visualize both the value surface of the critic and the samples obtained from $p_{D_\phi}$ with Langevin dynamics. The results are shown in Figure.~\ref{fig:toy}. As can be observed, the original WGAN (Fig.~\ref{subfig:toy:sngan}) is strong enough to cover most modes, but there are still some spurious links between two different modes. The enhanced distribution $p_{D_\phi}$ (Fig.~\ref{subfig:toy:sngan-dcd}), however, has the ability to reduce spurious links and recovers the modes with underestimated density. More precisely, after the MCMC fine-tuning procedure (Fig.~\ref{subfig:toy:sngan-dcd}), the gradients of the value surface become more meaningful so that all the regions with high density in data distribution $p_{\text{data}}$ are assigned with high $D_\phi$ value, \emph{i.e.}, lower energy($\exp(-D_\phi)$). By contrast, in the original discriminator (Fig.~\ref{subfig:toy:sngan}), the lower energy regions in $p_{D_\phi}$ are not necessarily consistent with the high-density region of $p_{\text{data}}$.

\subsection{Real-World Image Generation}

To quantitatively and empirically study the proposed DCD approach, in this section, we conduct experiments on unsupervised real-world image generation with DCD and its related counterparts. On several commonly used image datasets, experiments demonstrate that our proposed DCD algorithm can always achieve better performance on different benchmarks with a significant margin.

\subsubsection{Experimental setup}

\textbf{Baselines.} 
We evaluated the following models as our baselines: we take PixelCNN~\cite{van2016conditional}, PixelIQN~\cite{ostrovski2018autoregressive}, and MoLM~\cite{ravuri2018learning} as representatives of other types of generative models. For the energy-based model, we compared the proposed method with EBM~\cite{du2019implicit} and NCSN~\cite{song2019generative}. For GAN models, we take WGAN-GP~\cite{gulrajani2017improved}, Spectral Normalization GAN (SNGAN)~\cite{miyato2018spectral}, and Progressiv eGAN~\cite{karras2017progressive} for comparison. We also take the aforementioned DRS~\cite{azadi2018discriminator}, DOT~\cite{tanaka2019discriminator} and MH-GAN~\cite{turner2018metropolis} into consideration. The choices of EBM and GANs are due to their close relation to our proposed method, as analyzed in Section \ref{sec:methodology}. We omit other previous GAN methods since as a representative of a state-of-the-art GAN model, SNGAN and Progressive GAN has been shown to rival or outperform several former methods such as the original GAN \cite{goodfellow2014generative}, the energy-based generative adversarial network \cite{zhao2016energy}, and the original WGAN with weight clipping \cite{arjovsky2017wasserstein}.

\textbf{Evaluation Metrics.}
For evaluation, we concentrate on comparing the quality of generated images since it is well known that GAN models cannot perform reliable likelihood estimations \cite{theis2015note}. We choose to compare the Inception Scores \cite{salimans2016improved} and Frechet Inception Distances (FID) \cite{heusel2017gans} reached during training iterations, both computed from 50K samples. A high image quality corresponds to high Inception and low FID scores. 
Specifically, the intuition of IS is that high-quality images should lead to high confidence in classification, while FID aims to measure the computer-vision-specific similarity of generated images to real ones through Frechet distance.

\textbf{Data.}
We use CIFAR-10 \cite{krizhevsky2009learning} and STL-10 \cite{coates2011analysis}, which are all standard datasets widely used in generative literature. STL-10 consists of
unlabeled real-world color images, while CIFAR-10 is provided with class labels, which enables us to conduct conditional generation tasks. For STL-10, we also shrink the images into $32\times 32$ as in previous works. The pixel values of all images are rescaled into $[-1, 1]$.
%MNIST \cite{lecun1998gradient}, 

\textbf{Network Architecture.}
For all experiment settings, we follow Spectral Normalization GAN (SNGAN) \cite{miyato2018spectral} and adopt the same Residual Network (ResNet) \cite{he2016deep} structures and hyperparameters, which presently is the state-of-the-art implementation of WGAN. Details can be found in Appendix.~\ref{app:sec:network-arch}. We take their open-source code and pre-trained model as the base model for the experiments on CIFAR-10. For STL-10, since there is no pre-trained model available to reproduce the results, we train the SNGAN from scratch and take it as the base model.

\begin{figure*}[t]
\begin{minipage}{0.55\textwidth}
% \begin{table}[!t]%
\begin{center}
\begin{adjustbox}{max width=.9\linewidth}
\begin{tabular}{lcc}
        \toprule
        Model & Inception & FID\\
        \midrule
        \multicolumn{3}{l}{\textbf{CIFAR-10 Unconditional}} \\
        \midrule
        PixelCNN~\cite{van2016conditional} & $4.60$ & $65.93$\\
        PixelIQN~\cite{ostrovski2018autoregressive} & $5.29$ & $49.46$\\
        EBM~\cite{du2019implicit} & $6.02$ & $40.58$ \\
        WGAN-GP~\cite{gulrajani2017improved} & $7.86 \pm .07$ & $18.12$\\
        MoLM~\cite{ravuri2018learning} & $7.90\pm .10$ & $\mathbf{18.9}$\\
        SNGAN~\cite{miyato2018spectral} & $8.22\pm .05$ & $21.7$ \\
        ProgressiveGAN~\cite{karras2017progressive} & $8.80\pm .05$ & - \\
        NCSN~\cite{song2019generative} & {$8.87 \pm .12$} & $25.32$\\
        \midrule
        % DCGAN & $2.8789$ & -\\
        DCGAN w/ DRS(cal)~\cite{azadi2018discriminator} & $3.073 $ & - \\
        DCGAN w/ MH-GAN(cal)~\cite{turner2018metropolis} & $3.379 $ & - \\
        % ResNet-SAGAN w/o DOT & $7.85 \pm .11$ & $21.53$\\
        ResNet-SAGAN w/ DOT~\cite{tanaka2019discriminator} & $8.50 \pm .12$ & $19.71$\\
        \midrule
        \textbf{SNGAN-DCD (Pixel)} & {$8.54 \pm .11$} & $21.67$\\
        \textbf{SNGAN-DCD (Latent)} & {$\mathbf{9.11} \pm .04$} & $\mathbf{16.24}$\\
        \bottomrule
        \toprule
        \multicolumn{3}{l}{\textbf{CIFAR-10 Conditional}} \\
        \midrule
        EBM~\cite{du2019implicit} & $8.30$ & $37.9$ \\
        SNGAN~\cite{miyato2018spectral} & $8.43 \pm .09$ & $15.43$\\
        \textbf{SNGAN-DCD (Pixel)} & {$8.73 \pm .13$} & $22.84$\\
        \textbf{SNGAN-DCD (Latent)} & {$8.81 \pm .11$} & $15.05$\\
        BigGAN~\cite{brock2018large} & $\mathbf{9.22}$ & $\mathbf{14.73}$\\
        \bottomrule
    \end{tabular}
\end{adjustbox}
\end{center}
\captionof{table}{Inception and FID scores for CIFAR-10.} \label{tab:score-cifar}
% \end{table}
\end{minipage}
\hspace{+5pt}
\begin{minipage}{0.4\textwidth}
    \centering
    \includegraphics[width=0.99\columnwidth]{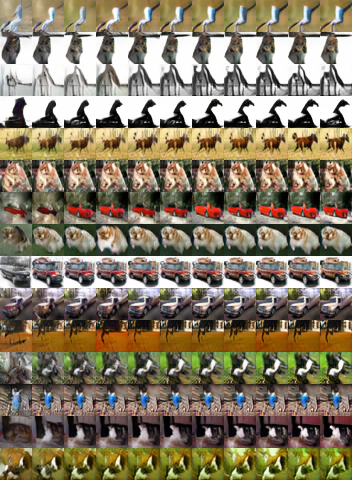}
    % \vspace{-19pt}
    \caption{Unconditional CIFAR-10 Langevin dynamics visualization.}
    \label{fig:mcmc-cifar}
\end{minipage}
\end{figure*}

\subsubsection{Results}

\begin{wraptable}{r}{0.5\textwidth}
\vspace{-10pt}
% \begin{table}[!t]%
\begin{center}
\begin{adjustbox}{max width=1.2\linewidth}
\begin{tabular}{lcc}
        \toprule
        Model & Inception & FID\\
        \midrule
        SNGAN~\cite{miyato2018spectral} & $8.90\pm .12$ & $18.73$ \\
        \textbf{SNGAN-DCD (Pixel)} & {$9.25 \pm .09$} & $22.25$\\
        \textbf{SNGAN-DCD (Latent)} & {$\mathbf{9.33} \pm .04$} & $17.68$\\
        \bottomrule
    \end{tabular} 
\end{adjustbox}
\end{center}
\caption{Inception and FID scores for STL-10} \label{tab:score-stl}
% \end{table}
\end{wraptable}

For quantitative evaluation, we report the inception score~\cite{salimans2016improved} and FID \cite{heusel2017gans} scores on CIFAR-10 in Tab.~\ref{tab:score-cifar} and STL-10 in Tab.~\ref{tab:score-stl}. As shown in the Tab.~\ref{tab:score-cifar}, in pixel space, by introducing the proposed DCD algorithm, we achieve a significant improvement of inception score over the SNGAN. The reported inception score is even higher than most values achieved by class-conditional generative models. Our FID score of $21.67$ on CIFAR-10 is competitive with other top generative models. When the DCD is conducted in the latent space, we further achieve a $9.11$ inception score and a $16.24$ FID, which is a new state-of-the-art performance of IS. When combined with label information to perform conditional generation, we further improve the FID to $15.05$, which is comparable with current state-of-the-art large-scale trained models~\cite{brock2018large}. Some visualization of generated examples can be found in Fig~\ref{fig:mcmc-cifar}, which demonstrates that the Markov chain is able to generate more realistic samples, suggesting that the MCMC process is meaningful and effective. Tab.~\ref{tab:score-stl} shows the performance on STL-10, which demonstrates that as a generalized method, DCD is not over-fitted to the specific data CIFAR-10. More experiment details and the generated samples of STL-10 can be found in Appendix.~\ref{app:sec:mcmc-stl}.

\section{Discussion and Future Work}
Based on the density ratio estimation perspective, the discriminator in $f$-GANs could be adapted to a wide range of application scenarios, such as mutual information estimation~\cite{hjelm2018learning} and bias correction of generative models~\cite{grover2019bias}.  However, as another important branch in GANs' research, the available information in WGANs discriminator is less explored. In this paper, we narrow down the scope of discussion and focus on the problem of how to leverage the discriminator of WGANs to further improve the sample quality in image generation. We conduct a comprehensive theoretical study on the informativeness of discriminator in different kinds of GANs. Motivated by the theoretical progress in the literature of WGANs, we investigate the possibility of turning the discriminator of WGANs into an energy function and propose a fine-tuning procedure of WGANs named as "discriminator contrastive divergence". The final image generation process is semi-amortized, where the generator acts as an initial state, and then several steps of Langevin dynamics are conducted.  We demonstrate the effectiveness of the proposed method on several tasks, including both synthetic and real-world image generation benchmarks. 

It should be noted that the semi-amortized generation allows a trade-off between the generation quality and sampling speed, which holds a slower sampling speed than a direct generation with a generator. Hence the proposed method is suitable to the application scenario where the generation quality is given vital importance.  Another interesting observation during the experiments is the discriminator contrastive divergence surprisingly reduces the occurrence of adversarial samples during training, so it should be a promising future direction to investigate the relationship between our method and bayesian adversarial learning. 

We hope our work helps shed some light on a generalized view to a method of connecting different GANs and energy-based models, which will stimulate more exploration into the potential of current deep generative models.

\newpage

\bibliography{99_bibliography}
% \bibliography{ref}
\bibliographystyle{plainnat}
% \section*{References}

\onecolumn
\appendix
\section{Proof of Theorem~\ref{gradient_direction}}\label{proof1}
% \subsection{Proof of Theorem \ref{theorem_gradient}}
% \label{lip_direction}
It should be noticed that Theorem.~\ref{gradient_direction} can be generalized to that Lipschitz continuity with $l_2$-norm (Euclidean Distance) can guarantee that the gradient is directly pointing towards some sample\cite{zhou2019lipschitz}. We introduce the following lemmas, and Theorem.~\ref{gradient_direction} is a special case.  

Let $(x, y)$ be such that $y\neq x $, and we define $x_t = x+ t\cdot (y-x)$ with $t \in [0,1]$. 

\begin{lemma}
If $f(x)$ is $k$-Lipschitz with respect to $\Vert . \Vert_p$ and $f(y)-f(x) = k\Vert y - x \Vert_p$, then $f(x_t) = f(x)+t\cdot k\Vert y - x \Vert_p$.
\end{lemma} 

\begin{proof}
As we know $f(x)$ is $k$-Lipschitz, with the property of norms, we have
\begin{align}\label{eq:linear_interopation}
f(y)-f(x) &= f(y)-f(x_t)+ f(x_t)-f(x) \nonumber \\
& \leq f(y)-f(x_t)+k\Vert x_t-x\Vert_p  = f(y)-f(x_t)+t\cdot k\Vert y - x\Vert_p \nonumber \\
& \leq k\Vert y-x_t\Vert_p+t\cdot k\Vert y - x\Vert_p  = k \cdot (1-t)\Vert y - x\Vert_p+t \cdot k\Vert y - x\Vert_p \nonumber \\
& = k \Vert y - x\Vert_p. 
\end{align}
$f(y)-f(x) = k\Vert y - x \Vert_p$ implies all the inequalities is equalities. Therefore, $f(x_t) = f(x)+t\cdot k\Vert y - x \Vert_p$. \qedhere
\end{proof}

\begin{lemma}
Let $v$ be the unit vector $\frac{y-x}{\Vert y -x \Vert_2}$. If $f(x_t) = f(x)+t\cdot k\Vert y - x \Vert_2$, then $\tpdv{f(x_t)}{v}$ equals to $k$. 
\end{lemma}

\begin{proof}
\begin{equation}
\begin{aligned}
\vspace{-20pt}
\tpdv{f(x_t)}{v} 
&= \lim\limits_{h\rightarrow0} \frac{f(x_t+hv)-f(x_t)}{h} =\lim\limits_{h\rightarrow0} \frac{f(x_t+h\frac{y-x}{\Vert y-x \Vert_2})-f(x_t)}{h} \\ 
& =\lim\limits_{h\rightarrow0}\frac{f(x_{t+\frac{h}{\Vert y-x \Vert_2}})-f(x_t)}{h}  =\lim\limits_{h\rightarrow0}\frac{\frac{h}{\Vert y-x \Vert_2}\cdot k\Vert y-x \Vert_2}{h}=k. \nonumber \qedhere
\end{aligned}
\end{equation}
\end{proof}
Then we derive the formal proof of Theorem~\ref{gradient_direction}. 

\begin{proof} $ $ 
Assume $p=2$, if $f(x)$ is $k$-Lipschitz with respect to $\Vert.\Vert_2$ and $f(x)$ is differentiable at $x_t$, then $\Vert \nabla f(x_t) \Vert_2 \leq k$.  Let $v$ be the unit vector $\frac{y-x}{\Vert y -x \Vert_2}$. We have 
\begin{align}
k^2 =k\tpdv{f(x_t)}{v} &=k\left<v,\nabla f(x_t)\right>= \left<kv, \nabla f(x_t) \right> \leq \Vert kv \Vert_2\Vert \nabla f(x_t) \Vert_2 = k^2.
\end{align}
Because the equality holds only when $\nabla f(x_t) = kv = k\frac{y-x}{\Vert y -x \Vert_2}$, we have that $\nabla f(x_t) = k\frac{y-x}{\Vert y -x \Vert_2}$.
\end{proof}

\section{Proof of Theorem~\ref{opt_fail}}\label{proof2}
Theorem.~\ref{opt_fail} states that following the following procedure as introduced in \cite{tanaka2019discriminator}, there is non-unique stationary distribution. The complete procedure is to find the following $y$ for $x \sim P_{G_\theta}$:
\begin{align}
    y^*= \argmin_{x} \{\| x- y\|_2- D(x)\}.
\end{align}
To find the corresponding $y ^*$, the following gradient based update is conducted:
\begin{align}
 \{x \leftarrow x - \epsilon  \nabla_{ x} \left\{ || x - y||_2 -  D(x) \right\}.
\end{align}
For all the points $x_t$ in the linear interpolation of $x$ and target $y^*$ as defined in the proof of Theorem~\ref{gradient_direction},
\begin{align}
    \nabla_{ x_t} \left\{ || x_t - y||_2 -  D( x_t) \right\} = \frac{y-x}{\Vert y -x \Vert_2} - \frac{y-x}{\Vert y -x \Vert_2} = 0,
\end{align}
which indicates all points in the linear interpolation  satisfy the stationary condition.

\section{Proof of Proposition~\ref{prop1}}\label{proof3}

% In the main text we stated that $\mathrm{D}_{\text{KL}}[q||p] \geq \mathrm{D}_{\text{KL}}[q_T||p]$ if $q$ is the stationary distribution of the kernel $\mathcal{K}$. This is a direct result of the following lemma, and we provide a proof from \cite{cover:itbook1991} for completeness.
Proposition.~\ref{prop1} is the direct result of the following Lemma.~\ref{lem:dd}. Following \cite{li2017approximate}, we provide the complete proof as following.

\begin{lemma}\label{lem:dd}
\cite{cover2012elements}
Let $q$ and $r$ be two distributions for $\bm{z}_0$. Let $q_t$ and $r_t$ be the corresponded distributions of state $\bm{z}_t$ at time $t$, induced by the transition kernel $\mathcal{K}$. Then $\mathrm{D}_{\text{KL}}[q_t||r_t] \geq \mathrm{D}_{\text{KL}}[q_{t+1}||r_{t+1}]$ for all $t \geq 0$.
\end{lemma}
\begin{proof}
%Define joint distribution $\tilde q = q(x)\k(x | x')$
\begin{equation*}
\begin{aligned}
\mathrm{D}_{\text{KL}}[q_t||r_t] &= \mathbb{E}_{q_t}\left[ \log \frac{q_t(\bm{z}_t)}{r_t(\bm{z_t})} \right] \\
&= \mathbb{E}_{q_t(\bm{z}_t) \mathcal{K}(\bm{z}_{t+1}|\bm{z}_t) }\left[ \log \frac{q_t(\bm{z}_t) \mathcal{K}(\bm{z}_{t+1}|\bm{z}_t) }{r_t(\bm{z}_t) \mathcal{K}(\bm{z}_{t+1}|\bm{z}_t)} \right] \\
&= \mathbb{E}_{q_{t+1}(\bm{z}_{t+1}) q_{t+1}(\bm{z}_t|\bm{z}_{t+1})}\left[ \log \frac{q_{t+1}(\bm{z}_{t+1}) q(\bm{z}_{t}|\bm{z}_{t+1}) }{r_{t+1}(\bm{z}_{t+1}) r(\bm{z}_{t}|\bm{z}_{t+1})} \right] \\
&= \mathrm{D}_{\text{KL}}[q_{t+1}||r_{t+1}] + \mathbb{E}_{q_{t+1}} \mathrm{D}_{\text{KL}}[q_{t+1}(\bm{z}_t | \bm{z}_{t+1}) || r_{t+1}(\bm{z}_t | \bm{z}_{t+1})].
\end{aligned}
\end{equation*}
\end{proof}

% \newpage

\section{Network architectures}
\label{app:sec:network-arch}

ResNet architectures for CIFAR-10 and STL-10 datasets. We use similar architectures to the ones used in~\cite{gulrajani2017improved}.

\begin{table}[!h]
\centering
\begin{tabular}{c}
    \toprule
    \midrule
    $z\in \mathbb{R}^{128} \sim \mathcal{N}(0, I)$ \\
    \midrule
    dense, $4 \times 4 \times 256$ \\
    \midrule
    ResBlock up 256\\
    \midrule
    ResBlock up 256\\
    \midrule
    ResBlock up 256\\
    \midrule
    BN, ReLU, 3$\times$3 conv, 3 Tanh\\
    \midrule
    \bottomrule
\end{tabular}
\vspace{+5pt}
\caption{Generator}
\end{table}

\begin{table}[!h]
\centering
\begin{tabular}{c}
    \toprule
    \midrule
    RGB image $x\in \mathbb{R}^{32\times 32 \times 3}$ \\
    \midrule
    ResBlock down 128\\
    \midrule
    ResBlock down 128\\
    \midrule
    ResBlock 128\\
    \midrule
    ResBlock 128\\
    \midrule
    ReLU\\
    \midrule
    Global sum pooling\\
    \midrule
    dense $\rightarrow$ 1\\
    \midrule
    \bottomrule
\end{tabular}
\vspace{+5pt}
\caption{Discriminator}
\end{table}

% \section{Additional results on ImageNet}
% \begin{table}[]
%     \centering
%     \begin{tabular}{c|c|c}
%     Setting & IS & FID \\
%     baseline & 21.11 & 63.24 \\
%     \hline
%     num-steps8noise-scale0.008step-lr1.0proj-norm0.001anearlingFalse & 20.67 & 86.87 \\
%     num-steps8noise-scale0.008step-lr5.0proj-norm0.001anearlingFalse & 21.03 & 77.20 \\
%     num-steps8noise-scale0.0008step-lr10.0proj-norm0.001anearlingFalse & \textbf{21.53} & \textbf{64.96} \\
%     num-steps8noise-scale0.008step-lr10.0proj-norm0.0001anearlingFalse & 20.73 & 85.86 \\
%     num-steps8noise-scale0.008step-lr10.0proj-norm0.001anearlingFalse & 21.39 & 75.81 \\
%     num-steps8noise-scale0.008step-lr10.0proj-norm0.01anearlingFalse & 21.06 & 76.21 \\
%     num-steps8noise-scale0.08step-lr10.0proj-norm0.0001anearlingFalse & 15.96 & 118.01 \\
%     num-steps8noise-scale0.0008step-lr10.0proj-norm0.001anearlingTrue & 21.08 & 76.36 \\
%     num-steps8noise-scale0.008step-lr10.0proj-norm0.001anearlingTrue & 20.73 & 81.72 \\
%     num-steps8noise-scale0.008step-lr100.0proj-norm0.0001anearlingFalse & 20.81 & 86.92 \\
%     num-steps20noise-scale0.008step-lr10.0proj-norm0.001anearlingFalse & 20.79 & 73.22
%     \end{tabular}
%     \caption{Caption}
%     \label{tab:my_label}
% \end{table}

\section{Discussions on Objective Functions}
\label{app:obj}
Optimization of the standard objective of WGAN, \emph{i.e.} with $r(x) = m(x) = x$ in Eq.~\ref{gan_obj}, are found to be unstable  due to the numerical issues and free offset~\cite{zhou2019lipschitz,miyato2018spectral}. Instead, several surrogate losses are actually used in practice. For example, the logistic loss($r(x) = m(x) = -\log(1+e^{-x})$) and hinge loss($r(x) = m(x) = \min (0,x)$) are two widely applied objectives. Such surrogate losses are valid due to that they are actually the lower bounds of the Wasserstain distance between the two distributions of interest. The statement can be easily derived by the fact that $-\log(1+e^{-x}) \leq x$ and $\min (0,x) \leq x$. A more detailed discussion could also be found in \cite{tanaka2019discriminator}.

Note that $\min (0,-1+x)$ and $-\log(1+e^{-x})$ are in the function family proposed in \cite{zhou2019lipschitz}, and Theorem 4 in \cite{zhou2019lipschitz} guarantees the gradient property of discriminator.

\section{More Experiment Details}
\label{app:sec:mcmc-stl}

\subsection{CIFAR-10}
For the meta-parameters in DCD Algorithm~\ref{DCD}, when the MCMC process is conducted in the pixel space, we choose $6-8$ as the number of MCMC steps $K$, and set the step size $\epsilon$ as $10$ and the standard deviation of the Gaussian noise as $0.01$, while for the latent space we set $K$ as $50$, $\epsilon$ as $0.2$ and the deviation as $0.1$. Adam optimizer \cite{kingma2014adam} is set with $2\times 10^{-4}$ learning rate with $\beta_1=0,\beta_2=0.9$. We use $5$ critic updates per generator update, and a batch size of $64$.

\subsection{STL-10}
We show generated samples of DCD during Langevin dynamics in Fig.~\ref{app:fig:stl}. We run 150 steps of MCMC steps
and plot generated sample for every 10 iterations. The step size is set as $0.05$ and the noise is set as $N(0,0.1)$.

\begin{figure}[!h]
    \centering
    \includegraphics[width=1.0\columnwidth]{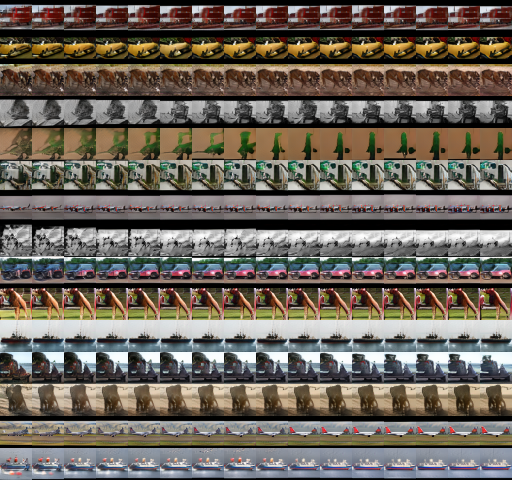}
    \caption{STL-10 Langevin dynamics visualization.}
    \label{app:fig:stl}
\end{figure}

\end{document}